\documentclass[runningheads]{llncs}
\usepackage{graphicx}
\usepackage{amsmath,amssymb} 
\usepackage{color}
\usepackage[width=122mm,left=12mm,paperwidth=146mm,height=193mm,top=12mm,paperheight=217mm]{geometry}
\begin{document}
\pagestyle{headings}
\mainmatter
\title{Order-Consistent Latent Geometry for Reliable Short-Horizon Planning in World Models}
\titlerunning{Self-Approaching Latent Geometry for Order-Consistent Planning}

\authorrunning{J. Wu et al.}

\author{Junlin Wu \and Wen Fan \and Dandan Zhang}
\institute{Imperial College London, London, United Kingdom}

\maketitle
\begin{abstract}

Latent planners often rank candidate action sequences by terminal Euclidean
distance to a goal latent. Exact agreement with transition distance would make
this score ideal, but is stronger than planning requires: along a
transition-shortest rollout, the essential requirement is that latent distance
does not reverse true transition progress.

We show that this pathwise no-reversal requirement is precisely the classical
self-approaching condition. To the best of our knowledge, we are the first to
identify self-approachingness as a planner-facing compatibility condition for
terminal Euclidean latent planning and to connect its violation to an exact
finite-horizon candidate-order reversal.

Because visual world models predict discrete latent sequences, we analyze their
polygonal interpolation. We prove that an acute bend creates a nonempty region
in which terminal Euclidean distance ranks a transition-worse candidate ahead
of one that makes genuine progress, even under exact rollouts. We further
interpret latent dimension as angular capacity through a classical
spherical-code bound. A controlled six-branch experiment, including a
zero-padded three-dimensional baseline, validates the predicted ordering
failure and separates angular arrangement from ambient dimensionality.

\end{abstract}

\section{Introduction}
\label{sec:intro}

Latent world models have emerged as a promising framework for goal-directed
control. Rather than planning in observation space, they predict the latent
consequences of candidate action sequences and evaluate those candidates in a
learned representation. This paradigm spans latent control, differentiable
planning, and online model predictive control~\cite{watter2015e2c,srinivas2018upn,hafner2019planet,hansen2022tdmpc}.
Goal-conditioned and reconstruction-free systems similarly optimize toward a
latent image goal without decoding future pixels~\cite{kurutach2018causal,zhou2025dinowm,assran2025vjepa2,maes2026lewm}.

Reliable planning therefore depends not only on predictive accuracy, but also
on whether latent geometry induces the right ordering over candidate
trajectories. A world model may predict every rollout exactly and nevertheless
select an action that moves in the wrong direction when terminal Euclidean
distance ranks a transition-inferior candidate above one that makes genuine
progress. Thus perfect prediction can still be paired with a misleading
decision objective.

Short receding-horizon optimization is common, and several methods append a
learned terminal value to summarize beyond-horizon consequences~\cite{hoeller2020deepvalue,hansen2022tdmpc}.
We instead ask what raw terminal Euclidean goal matching can guarantee. Exact
agreement with transition distance is unnecessarily restrictive: on one
shortest rollout it forces a straight latent segment, while real transition
structures contain bends, obstacles, merges, and branches. Planning needs only
the weaker ordinal requirement that moving forward along a shortest rollout
must not make the same future goal appear farther away. We call this
\emph{pathwise transition-order consistency} and show that it is exactly
self-approachingness.

\paragraph{Relation to prior work.}
Planning-aware representations use structured transition models, learned
action distances, curvature regularization, and reachability-aware objectives
or metrics~\cite{kurutach2018causal,steccanella2025mad,wang2026straightening,li2026rcaux,li2026beyond}.
These works establish that raw latent proximity can be insufficient; we do not
claim that general observation. Classical greedy and self-approaching drawings
study distance-monotone geometric routing~\cite{papadimitriou2004greedy,icking1999self,alamdari2013self},
but not learned world models, horizon-truncated terminal objectives, or the
exact candidate-order reversal derived here.

\paragraph{Contributions.}
We make four principal contributions:
\begin{itemize}
\setlength{\itemsep}{0.5pt}
\setlength{\parsep}{0pt}
\setlength{\parskip}{0pt}
\setlength{\topsep}{1pt}
    \item We separate exact pathwise metric fidelity from the weaker ordering
    property needed by terminal-distance guidance, and prove that exact
    fidelity forces every shortest latent rollout to be straight.
    \item We identify self-approachingness as this ordering property and derive
    an equivalent differential forward-half-space condition.
    \item For locally straight mandatory junctions, we prove a sharp
    necessary-and-sufficient non-acute criterion and the resulting
    finite-horizon wrong-way optimum.
    \item We prove the angular-capacity bound $n\leq2D$ and validate the
    mechanism in a six-branch stress test with a padded-dimension control.
\end{itemize}

\section{Theory}
\label{sec:theory}

\subsection{Pathwise Transition-Order Consistency}
\label{sec:order}

Let $(\mathcal{X},d_{\mathrm{T}})$ be a transition metric space,
$\phi:\mathcal{X}\rightarrow\mathbb{R}^D$ a latent embedding, and
$\gamma:[0,L]\rightarrow\mathcal{X}$ a transition-shortest rollout
parameterized by transition length. We write
$\eta=\phi\circ\gamma$ for its continuous latent image.

Requiring latent Euclidean distance to reproduce transition distance exactly is
unnecessarily restrictive. In particular, the equality
$|\eta(t)-\eta(s)|_2=t-s$ for every $s<t$ forces $\eta$ to be a unit-speed
straight segment by the equality condition of the Euclidean triangle
inequality. Planning only requires a weaker ordinal property. Since $\gamma$ is
transition-shortest, moving from $\gamma(s)$ to $\gamma(t)$ for $s<t<r$
strictly decreases the true remaining transition distance to the future state
$\gamma(r)$. Latent distance should not reverse this progress.

\begin{definition}[Pathwise transition-order consistency]
\label{def:order}
A transition-shortest latent rollout $\eta$ is
\emph{pathwise transition-order consistent} if, for all
$0\leq s<t<r\leq L$,
\begin{equation}
|\eta(t)-\eta(r)|_2
\leq
|\eta(s)-\eta(r)|_2.
\label{eq:order-consistency}
\end{equation}
\end{definition}

This condition is precisely self-approachingness. It preserves the ordering of
true transition progress along a shortest rollout without requiring latent
Euclidean distance to reproduce transition distance numerically. For a
differentiable rollout, the same condition can be written as
$\dot{\eta}(t)^\top(\eta(r)-\eta(t))\geq0$ for every $r>t$. Thus nonzero
curvature is not itself a failure; a reversal occurs only when a future point
falls behind the current tangent half-space.

\subsection{Polygonal Rollouts and Acute-Bend Reversal}
\label{sec:reversal}

A visual world model produces a discrete latent rollout
$P=(z_0,z_1,\ldots,z_H)$. We analyze its polygonal interpolation, obtained by
joining consecutive predicted states with straight segments. This construction
does not assume that the environment or its dynamics are globally linear. It is
the minimal continuous geometric representation of a discrete model rollout.
Analyzing the interpolation rather than only its sampled vertices also makes
the condition independent of temporal resolution: an ordering reversal between
two predictions should not disappear merely because the rollout is sampled
more coarsely.

Consider a local two-segment portion of a transition-shortest rollout,
$z_\star+r u\rightarrow z_\star\rightarrow z_\star+b v$, where $r,b>0$
are transition distances from the turning state and
$u,v\in\mathbb{S}^{D-1}$ are the two outward directions. The correct rollout
moves inward along $-u$ and then outward along $v$, with true remaining
transition distance $r+b$.

\begin{theorem}[Sharp non-acute criterion]
\label{thm:criterion}
The two-segment shortest rollout is pathwise transition-order consistent for
every source and goal radius if and only if
\begin{equation}
u^\top v\leq0.
\label{eq:nonacute}
\end{equation}
Thus right and obtuse outward angles preserve pathwise ordering, whereas an
acute outward angle produces a strict reversal.
\end{theorem}

\begin{proof}
Pairs of points on the same straight segment are automatically
order-consistent. For a current point $z_\star+r u$ on the first segment and a
future point $z_\star+b v$ on the second, their squared latent distance is
$F(r,b)=r^2+b^2-2rb,u^\top v$. Correct progress toward $z_\star$ decreases
$r$, while $\partial F/\partial r=2(r-bu^\top v)$. This derivative is
nonnegative for every $r,b>0$ exactly when $u^\top v\leq0$. If
$u^\top v>0$, every $0<r<bu^\top v$ gives a strict ordering reversal.
\end{proof}

The same geometry can reverse the ordering of feasible finite-horizon
candidates. Assume the incoming transition is locally reversible. From
$z_\star+r u$, compare an inward candidate
$z_{\mathrm{in}}=z_\star+(r-h)u$ with a backward candidate
$z_{\mathrm{back}}=z_\star+(r+h)u$, where $0<h<r$. The inward candidate has
strictly smaller remaining transition distance to the goal
$z_\star+bv$.

\begin{proposition}[Exact candidate-order reversal]
\label{prop:pair-reversal}
Let $c=u^\top v$. If $c>0$ and $r<bc$, terminal Euclidean scoring strictly
ranks the transition-worse backward candidate ahead of the
transition-better inward candidate:
\begin{equation}
|z_{\mathrm{back}}-(z_\star+bv)|*2^2
-
|z*{\mathrm{in}}-(z_\star+bv)|_2^2
=
4h(r-bc)
<0.
\label{eq:pair-reversal}
\end{equation}
\end{proposition}

This result does not require the two candidates to be the planner's only
options. It establishes a strict pairwise inversion: a candidate with larger
true remaining transition distance receives a better terminal Euclidean score.
Therefore, the score fails to preserve the planning-relevant candidate
ordering even if a third candidate is ultimately selected from the full
candidate set.

In the branch-local environment studied in Sec.~\ref{sec:experiment}, the
short-horizon reachable set remains on the same branch. In that special case,
the complete reachable interval is $q\in[r-H,r+H]$, and the Euclidean terminal
loss is minimized at the projection of $bc$ onto this interval. Consequently,
when $c>0$ and $r<bc$, the globally selected endpoint satisfies $q>r$ and the
planner moves away from the turning state. Proposition~\ref{prop:pair-reversal}
is more general: it proves a genuine ordering failure without assuming that
motion is restricted to two directions.

Both results assume exact candidate rollouts. The failure therefore arises from
terminal scoring geometry rather than accumulated prediction error or model
hallucination.

\subsection{Latent Dimension as Angular Capacity}
\label{sec:capacity}

The preceding result gives a planner-facing interpretation of latent dimension.
Suppose $n$ locally distinct transition branches meet at a common state, with
outward latent directions $u_1,\ldots,u_n\in\mathbb{S}^{D-1}$, and every pair
of branches may form a transition-shortest two-segment rollout. By
Theorem~\ref{thm:criterion}, preserving pathwise transition order for every
branch pair requires $u_i^\top u_j\leq0$ whenever $i\neq j$.

\begin{corollary}[Angular capacity]
\label{cor:capacity}
If $n$ transition branches must be pairwise non-acute in $\mathbb{R}^D$, then
\begin{equation}
n\leq2D,
\qquad\text{or equivalently}\qquad
D\geq\left\lceil\frac{n}{2}\right\rceil.
\label{eq:capacity}
\end{equation}
The bound is tight and is achieved by subsets of
${\pm e_1,\ldots,\pm e_D}$.
\end{corollary}

\begin{proof}
The unit branch directions form a spherical code with minimum angle at least
$\pi/2$. Rankin's bound gives $A(D,\pi/2)=2D$~\cite{rankin1955caps}, and the
cross-polytope achieves equality.
\end{proof}

Consequently, latent dimension provides angular capacity: increasing $D$
allows more transition-distinct branches to remain geometrically separated
without introducing acute configurations that reverse pathwise progress
ordering.

\section{Controlled Six-Branch Study}
\label{sec:experiment}

We isolate the theorem on a six-branch star. A state is $(i,r)$ with branch
$i\in\{0,\ldots,5\}$ and radius $r\in[0,1]$; cross-branch transition distance
is $r+b$. Every embedding uses $\phi(i,r)=r u_i$. We compare (i) a 2D
regular hexagon, (ii) the same hexagon zero-padded into 3D, and (iii) the 3D
cross-polytope $\{\pm e_1,\pm e_2,\pm e_3\}$. The first two embeddings have
identical pairwise angles, six acute pairs, and acute loss
$\mathcal{L}_{\mathrm{acute}}=\sum_{i<k}[u_i^{\top}u_k]_{+}^{2}=1.5$; the
cross-polytope has no acute pair and zero acute loss.

\begin{figure}[t]
\centering
\includegraphics[width=\linewidth]{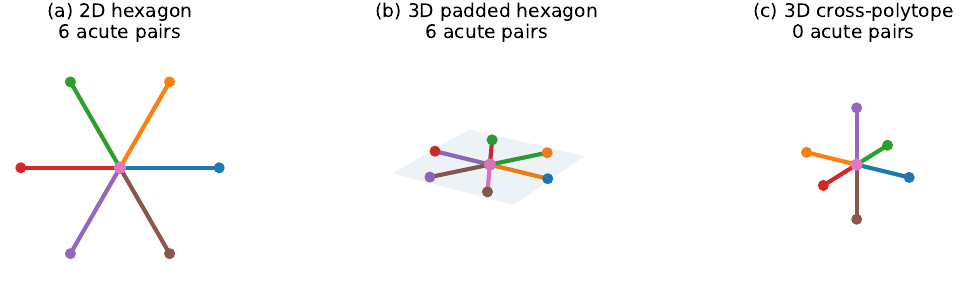}
\caption{
Geometry of the six-branch latent embeddings.
\textbf{(a)} A two-dimensional regular hexagon, in which adjacent branch
directions form acute angles.
\textbf{(b)} A three-dimensional zero-padded copy of the same hexagon, which
changes the ambient dimension but preserves all pairwise angles.
\textbf{(c)} A three-dimensional cross-polytope embedding with directions
$\{\pm e_1,\pm e_2,\pm e_3\}$, in which all branch pairs are non-acute.
The comparison separates angular arrangement from ambient dimension alone.
}
\label{fig:geometry}
\end{figure}
\paragraph{Protocol.}
We discretize each branch at $0.05$. Across all ordered branch pairs and
start/goal radii in $\{0.1,\ldots,1.0\}$ ($3{,}000$ queries per horizon), the
planner enumerates terminal states within budget $H$, minimizes the terminal
Euclidean score $\|\phi(y)-\phi(g)\|_2^2$, executes one graph step, and
replans. We report
strict outward-over-inward misranking and success within 100 steps. Ties are
excluded from strict misranking and use a fixed lexicographic rule in closed
loop. The released script reproduces all curves.

\begin{figure}[t]
\centering
\begin{minipage}[t]{0.49\linewidth}
\centering
\includegraphics[width=\linewidth]{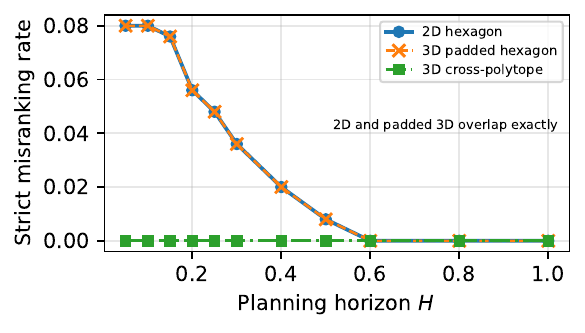}
\vspace{-1mm}
\small (a) Strict terminal-score misranking
\end{minipage}\hfill
\begin{minipage}[t]{0.49\linewidth}
\centering
\includegraphics[width=\linewidth]{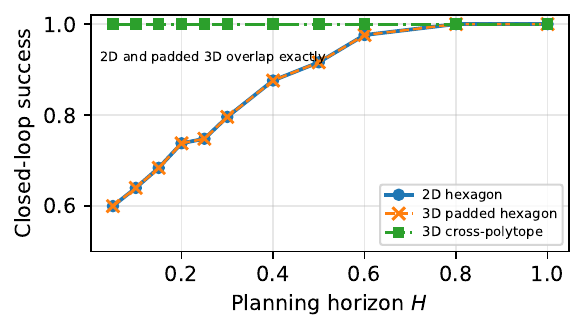}
\vspace{-1mm}
\small (b) Receding-horizon success
\end{minipage}
\caption{
\textbf{(a)} Strict terminal-score misranking across planning horizons.
\textbf{(b)} Receding-horizon planning success.
All three embedding conditions are plotted. The 2D and zero-padded 3D
hexagons overlap exactly because zero-padding preserves their pairwise angular
geometry. Only the reoriented 3D cross-polytope removes acute conflicts,
strict misrankings, and the associated planning failures.
}
\label{fig:results}
\end{figure}

\paragraph{Results.}
At $H=0.05$, the 2D and padded embeddings produce a strict misranking rate of
$8.0\%$ and closed-loop success of $60.0\%$; the cross-polytope has zero strict
misranking and $100\%$ success. At $H=0.30$, the acute embeddings improve to
$3.6\%$ misranking and $79.6\%$ success, while the cross-polytope remains at
zero and $100\%$. The acute error vanishes once the horizon is long enough to
see beyond the local trap. Most importantly, the 3D padded control is identical
to 2D at every horizon, whereas the 3D cross-polytope succeeds: the gain comes
from available angular arrangement, not coordinate count by itself.

\section{Scope and Conclusion}
\label{sec:conclusion}

Terminal-distance planning needs no global isometry; it needs Euclidean
scores not to reverse progress along the rollout prefixes on which the planner
acts. This condition is exactly self-approachingness. Acute mandatory junctions
create exact wrong-way optima, while multi-branch bottlenecks turn latent
width into finite junction-ordering capacity. Our guarantees are pathwise and
apply to graph-like point bottlenecks with locally reversible branches; they do
not ensure global ranking across unrelated paths or imply impossibility for
continuous-action tasks.

\bibliographystyle{splncs}
\bibliography{egbib}
\end{document}